\documentclass[9pt,twocolumn]{article}
\usepackage[left=.75in,right=.75in]{geometry}
\usepackage{booktabs}
\usepackage{dsfont}
\usepackage{amssymb}

\usepackage{multirow}
\usepackage{tikz}
\usepackage{amsmath}
\usepackage{xurl}
\usepackage{threeparttable}
\usepackage{multirow}
\usepackage{placeins}
\usepackage{subcaption}

\usepackage{xcolor}
\usepackage[all]{xy}
\usepackage{hyphenat}
\usepackage{tablefootnote}
\usepackage{titling}
\usepackage{balance}

\usepackage[T1]{fontenc}
\usepackage{amsfonts}       % blackboard math symbols
\usepackage{amsmath}
\usepackage{amsthm}
\newtheoremstyle{tightdefinition} % Name of style
  {6pt}   % Space above
  {6pt}   % Space below
  {\normalfont}  % Body font
  {}      % Indent amount
  {\bfseries} % Theorem head font
  {.}     % Punctuation after theorem head
  {0.5em} % Space after theorem head
  {\thmname{#1}~\thmnumber{#2}\thmnote{~(#3)}} % Theorem head spec

\theoremstyle{tightdefinition}
\newtheorem{definition}{Definition}
\usepackage{amssymb}
\usepackage{thmtools}
\usepackage{bm}
\usepackage{bbm}
\usepackage{graphicx}
\usepackage[numbers, square, sort&compress]{natbib}
\usepackage{hyperref}
\usepackage{xcolor}
\usepackage{booktabs}
\usepackage{multicol}
\usepackage{multirow}
\usepackage[capitalize]{cleveref}
\usepackage{float}

\usepackage{algorithm,algpseudocode}

\algnewcommand{\Inputs}[1]{%
  \State \textbf{Inputs:}
  \Statex \hspace*{\algorithmicindent}\parbox[t]{\linewidth}{\raggedright #1}
}
\algnewcommand{\Output}[1]{%
  \State \textbf{Output:}
  \Statex \hspace*{\algorithmicindent}\parbox[t]{\linewidth}{\raggedright #1}
}
\algnewcommand{\Initialize}[1]{%
  \State \textbf{Initialize:}
  \Statex \hspace*{\algorithmicindent}\parbox[t]{\linewidth}{\raggedright #1}
}
\algnewcommand{\comm}[1]{ {\ttfamily\textcolor{blue}{// #1}} }

\pretitle{\centering\Large\bfseries}
\posttitle{\par\vspace{1em}}

\preauthor{\centering\normalsize\lineskip 0.5em}
\postauthor{\par}

\predate{\centering\small}
\postdate{\par\vspace{1.5em}}

\usepackage{amsthm}
\usepackage{thmtools}
\usepackage{thm-restate}

\begin{document}

\newcommand{\figsize}{0.7\linewidth}

\newcommand{\dataspace}{\mathbb{D}}
\newcommand{\datagen}{\mathcal{D}}
\newcommand{\model}{\datagen_{\theta}}
\newcommand{\dt}{D}
\newcommand{\syndata}{D_\mathrm{syn}}
\newcommand{\mt}{\theta} %{\datagen_{\theta}}%{\theta}
\newcommand{\parameter}{\theta}%{\theta}
\newcommand{\modelspace}{\Theta}
\newcommand{\train}{\mathcal{A}}
\newcommand{\metric}{\psi}
\newcommand{\ravg}{R^\text{T}}%{R_\text{trad}^{\metric}}
\newcommand{\ravgauc}{R_\text{avg}^{AUC}}
\newcommand{\ravgfpr}{R_\text{avg}^{TPR}}
\newcommand{\rsp}{R^\text{MS}}%{R_\text{ms}^{\metric}}
\newcommand{\rspauc}{R_\text{ms}^{AUC}}
\newcommand{\rspfpr}{R_\text{ms}^{TPR}}
\newcommand{\xt}{x}
\newcommand{\attack}{\phi}
\newcommand{\score}{s}
\newcommand{\todo}[1]{\textbf{\textcolor{red}{quad TODO: #1}}}
\newcommand{\E}{\mathbb{E}}
\newcommand{\alphasp}{\hat \alpha^\text{MS}_\phi}
\newcommand{\alphaavg}{\hat \alpha^\text{T}_\phi}
\newcommand{\betasp}{\hat \beta^\text{MS}_\phi}
\newcommand{\betaavg}{\hat \beta^\text{T}_\phi}

\title{Lost in the Averages: Reassessing Record-Specific Privacy Risk Evaluation\thanks{This is an extended version of the paper published at the Data Privacy Management (DPM) workshop at ESORICS 2025.}}

\date{}

\author{
    Nata\v sa Kr\v co\textsuperscript{1} \quad
    Florent Gu\' epin\textsuperscript{1} \quad
    Matthieu Meeus\textsuperscript{1} \quad
    Bogdan Kulynych\textsuperscript{2} \quad
    Yves-Alexandre de Montjoye\textsuperscript{1} \\
    \vspace{0.5em}
    \textsuperscript{1}Imperial College London \\
    \textsuperscript{2}Lausanne University Hospital (CHUV), Lausanne, Switzerland
}

\maketitle
\begin{abstract}
  Synthetic data generators and machine learning models can memorize their training data, posing privacy concerns. Membership inference attacks (MIAs) are a standard method of estimating the privacy risk of these systems. The risk of individual records is typically computed by evaluating MIAs in a record-specific privacy game. We analyze the record-specific privacy game commonly used for evaluating attackers under realistic assumptions (the \textit{traditional} game)---particularly for synthetic tabular data---and show that it averages a record's privacy risk across datasets. We show this implicitly assumes the dataset a record is part of has no impact on the record's risk, providing a misleading risk estimate when a specific model or synthetic dataset is released. Instead, we propose a novel use of the leave-one-out game, used in existing work exclusively to audit differential privacy guarantees, and call this the \textit{model-seeded} game. We formalize it and show that it provides an accurate estimate of the privacy risk posed by a given adversary for a record in its specific dataset. We instantiate and evaluate the state-of-the-art MIA for synthetic data generators in the traditional and model-seeded privacy games, and show across multiple datasets and models that the two privacy games indeed result in different risk scores, with up to 94\% of high-risk records being overlooked by the traditional game. We further show that records in smaller datasets and models not protected by strong differential privacy guarantees tend to have a larger gap between risk estimates. Taken together, our results show that the model-seeded setup yields a risk estimate specific to a certain model or synthetic dataset released and in line with the standard notion of privacy leakage from prior work, meaningfully different from the dataset-averaged risk provided by the traditional privacy game.
\end{abstract}

\section{Introduction}\label{sec:introduction}

Models ranging from synthetic data generators (SDGs) to machine learning (ML) models have been shown to memorize their training data, potentially allowing attackers to tell whether specific records were used for training \cite{stadler2022synthetic,houssiau2022tapas,meeus2023achilles,annamalai2023linear,hayes2019logan,pyrgelis2018knock,guan2024zero} or even reconstruct entire training examples \cite{balle2022reconstructing,yang2019neural,wang2019beyond,he2019modelinversion}. As models are increasingly trained on personal and sensitive data---particularly in domains such as healthcare, law, and finance \cite{osuala2025medicine, chhikara2025predictive, cao5126964assessing}---concerns about their implications for privacy continue to grow.

Membership inference attacks (MIAs) have become the standard approach for empirically estimating the privacy risk of synthetic data and ML models~\cite{kumar2020mlprivacy, shokri2017membership, carlini2022membership, stadler2022synthetic, pollock2024free}. MIAs aim to determine whether a target record was included in the training dataset of a given model. They can pose a direct privacy risk, and also provide an upper bound on the performance of other attacks such as attribute inference or data reconstruction~\cite{salem2023sok}. MIAs can be developed under varying assumptions, ranging from black-box access to the target model and no knowledge of the training dataset, to very strong attackers leveraging white-box access to the model and knowledge of all training records but the target. 

MIAs are evaluated in a controlled privacy game between an attacker and a data owner~\cite{yeom2018privacy, stadler2022synthetic}. We here study the record-specific privacy games used in existing literature which estimate how well an attacker can distinguish between models trained on the target record and those not. Record-specific privacy games are most often used in setups where the state-of-the-art attacks leverage record-specific information, such as for synthetic data generators~\cite{meeus2023achilles, houssiau2022tapas, stadler2022synthetic}, and for auditing formal privacy guarantees~\cite{annamalai2024theory}. In contrast, model-specific privacy games estimate the ability of an attacker to distinguish between records used to train a target model and those not. This type of privacy game is often used to evaluate MIAs against ML models~\cite{carlini2022membership, zarifzadeh2024lowcosthighpowermembershipinference, song2020systematicevaluationprivacyrisks, sablayrolles2019white, hayes2019logan, choquette2021label}.

\paragraph{Contributions.}

We analyze the \textit{traditional} privacy game commonly used to evaluate record-specific MIAs under realistic attacker assumptions. We show that, by using dataset sampling as a source of randomness, it averages the risk across datasets, implicitly assuming that a record's privacy risk is independent of the dataset it belongs to.%, and averages the risk of a record across datasets. 

We instead formalize and propose a novel use of the leave-one-out game, here called the \textit{model-seeded} privacy game, to evaluate an MIA under realistic attacker assumptions. This approach is consistent with the standard notion of differential privacy~\cite{dwork2006differential,dwork2014algorithmic}, which captures a record’s risk with respect to a specific dataset. Unlike the traditional game, we fix the target dataset and use only the model seed as a source of randomness. We show that the attack success rates computed using this privacy game converge to the record's \textit{differential privacy distinguisher} (DPD) risk---consistent with the standard notion of privacy leakage in existing literature---whereas the traditional game results in a dataset-independent estimate.

We instantiate the state-of-the-art record-specific MIA for synthetic data and evaluate it in both the traditional and model-seeded privacy game across $2$ datasets and $2$ synthetic data generators. We observe significant differences between the risk estimates given by the model-seeded and traditional privacy games. For instance, $94\%$ of high-risk records are misidentified by the traditional privacy game for the Adult dataset and Synthpop generator.
% , and the two risk estimates exhibit a root mean squared difference ($\mathrm{RMSD}$) of $0.07$.
We obtain similar results across experimental setups.

Finally, we show the gap between the traditional and model-seeded risk estimates to be generally higher for small and medium datasets (fewer than $10,000$ records), as often used in tabular synthetic data~\cite{stadler2022synthetic, meeus2023achilles, guepin2023synthetic}, and lower for large datasets, typically used for ML tasks~\cite{carlini2022membership}. We show that training with differential privacy with very strict $\varepsilon$ decreases the difference between traditional and model-seeded risk, likely due to overall lower MIA performance.

Taken together, we show that the commonly used privacy game for evaluating record-specific privacy risks can yield misleading estimates by averaging the risk across datasets. To address this, we propose to use the model-seeded privacy game, which provides more accurate risk estimates, aligning with differential privacy, critical for synthetic data or ML models trained on increasingly sensitive data.  

\section{Background}\label{sec:preliminary}
\subsection{Synthetic data generation}

We consider the setting of statistical learning over the space of records $\dataspace \subseteq \mathbb{R}^d$. We say that there exists a probability distribution $\datagen$ over $\dataspace$. We consider a dataset $D \in \dataspace^n$ that is i.i.d. sampled from the distribution: $D \sim \datagen^n$. Using dataset $D$, we train a model parameterized by $\parameter \in \modelspace$ using a training algorithm $\train: 2^{\dataspace} \rightarrow \modelspace$, where $2^{\dataspace}$ denotes the powerset of $\dataspace$. We use the term \emph{model} and its associated paramater vector $\parameter$ interchangeably.

A synthetic data generator (SDG) attempts to define a probability distribution that mimics statistical properties of the original data distribution $\datagen$.  A trained SDG, denoted as \(\model\) and parameterized with $\theta \leftarrow \train(D)$, can be used to generate a synthetic dataset \( \syndata \sim \model^n \). In the rest of the paper, we assume \(\syndata\) to have equal size as the training dataset used to train the model \( D \), i.e., \( |D| = |\syndata| = n \).

Practical SDGs can take various forms, such as probabilistic models (e.g., Bayesian networks \cite{baynet}) or deep generative models (e.g., generative adversarial networks (GANs) \cite{xu2019modeling}). In this work, we focus specifically on SDGs for tabular data \cite{baynet,synthpop} as attacks for this setting are inherently record-specific, making correct record-level evaluation particularly relevant (see~\cref{subsec:background_mia} for a more detailed review of attack development and methodology).

\subsection{MIA development}\label{subsec:background_mia}

For a \emph{target model} \(\mt \leftarrow \train(\dt)\), an MIA aims to infer whether a target record $\xt$ was included in $\dt$ (member) or not (non-member). For a fixed target record $\xt$, we denote by $\attack_{\xt}: \modelspace \rightarrow \{0, 1\}$ a MIA against target record $\xt$ and target model $\mt$. As we consider record-specific attacks, we drop the subscript $\xt$ when the target record is clear from context. In the rest of the section, we elaborate on the threat model considered for MIA development, the shadow modeling technique, and how an attacker can compute a membership score.

\paragraph{Threat model.}  
By \textit{threat model}, we refer to the assumptions made about the attacker's capabilities. We distinguish between the following assumptions:

\begin{enumerate}
    \item Dataset-level: the data the attacker has access to. The attacker may have no access to any real data from \(\mathcal{D}\)~\cite{guepin2023synthetic}, access to data drawn from the same distribution as \(\dt\)~\cite{stadler2022synthetic, houssiau2022tapas, meeus2023achilles}, or full access to \(\dt\) except for the knowledge of membership of \(\xt\), as considered in the threat model for the strong differential privacy attacker~\cite{jagielski2020auditingdifferentiallyprivatemachine,annamalai2024theory}.
    \item Model-level: access the attacker has to the trained model and their knowledge of the corresponding training configurations (e.g. model architecture, training hyperparameters). Model access can range from  black-box (query access) \cite{shokri2017membership,carlini2022membership} to white-box (full parameter access) ~\cite{nasr2019comprehensive,sablayrolles2019white,cretu2023investigating}.
\end{enumerate}

In this work, we adopt the standard assumptions for a \textbf{realistic record-specific attacker}~\cite{houssiau2022tapas, meeus2023achilles, guepin2023synthetic, stadler2022synthetic} in the context of tabular synthetic data. We assume the adversary has access to an \textit{auxiliary dataset} $D_{\text{aux}} \in 2^\dataspace$, which is drawn from the same distribution as \(\dt\) but does not overlap with it. For model access, we consider an attacker with black-box access to target model $\model$, able to query a synthetic dataset $\syndata$ sampled from $\model$, i.e. \(\syndata \sim \model^n\). Lastly, we assume the attacker to always have full knowledge of the exact training procedure $\train(\cdot)$ used to obtain \(\mt\) (e.g. model architecture, training hyperparameters).

\paragraph{Shadow modeling.}  
Shadow modeling is a technique used to develop MIAs by simulating the target model's training process. The attacker leverages their knowledge of the training data distribution (access to \( D_{\text{aux}} \)) to construct \textit{shadow datasets} $\{\smash{D_{\text{shadow}}^{(i)}} \mid i= 1 \ldots, N_{\text{shadow}}\}$ sampled from $D_{\text{aux}}$ and of the same size as \(\dt\). The attacker then explicitly designs `in' shadow datasets that include target record $\xt$ ($\xt \in \smash{D_{\text{shadow}}^{(i)}}$) and `out' shadow datasets that exclude it ($\xt \notin \smash{D_{\text{shadow}}^{(i)}}$). The attacker trains \textit{shadow models} $\{ \smash{\train(D_{\text{shadow}}^{(i)}}) \mid i=1,\ldots,N_{\text{shadow}} \}$ using the knowledge of the training procedure of the target model.
Thus, the attacker constructs a controlled set of models with known membership of the target record, which they can use to develop and refine the MIA.

\paragraph{Computing a membership score.} The membership prediction of an MIA is typically in the form of thresholding a \textit{membership score}  $\score_{\xt}: \modelspace \rightarrow \mathbb{R}$. We denote the attack as $\attack_{\xt}(\mt) = \mathbbm{1}[\score_{\xt}(\mt) \geq \gamma]$ for some given threshold $\gamma \in \mathbb{R}$.

Existing MIAs against SDGs under black-box access analyze the generated synthetic records, leveraging statistical measurements~\cite{stadler2022synthetic} and counting queries~\cite{meeus2023achilles, houssiau2022tapas} to extract features from the generated data. These are then used to train meta-classifiers to predict the membership of the target record. Notably, these membership scores are inherently \textit{record-specific}, driving the development and evaluation of MIAs tailored to individual records. 

In contrast, MIAs for predictive ML models are usually applicable across an arbitrary number of records. These attacks usually rely on the loss of the model on the target record~\cite{yeom2018privacy} to compute a membership score. Such scores can be computed for different records while remaining on a consistent scale, making them comparable across samples.

\subsection{Differential privacy and its hypothesis-testing interpretation}\label{sec:dp-risk}

Differential privacy (DP) is a formal privacy guarantee that limits the contribution of any single record in statistical learning. Under the classical definition by~\citet{dwork2014algorithmic}, a randomized training algorithm $\train(D)$ is differentially private if the inclusion or exclusion of any single record in $D$ will not significantly modify the resulting model distribution:
\begin{definition}\label{def:dp}
    A randomized training algorithm $\train(D)$ satisfies $(\varepsilon, \delta)$-DP if for any measurable subset $E$ of the model space $\modelspace$ and any partial dataset $\bar D$ and any record $x \in \dataspace$, we have:
    $$
    \begin{aligned}
    \Pr[\train(\bar D) \in E] & \leq e^\varepsilon \Pr[\train(\bar D \cup \{x\}) \in E] + \delta \\
     \Pr[\train(\bar D \cup \{x\}) \in E] & \leq e^\varepsilon \Pr[\train(\bar D) \in E] + \delta
     \end{aligned}
    $$
\end{definition}

This classical definition in terms of a difference between probability distributions has been shown to have an interpretation in terms of hypothesis testing~\cite{wasserman2010statistical,kairouz2015composition,dong2022gaussian}, or equivalently, in terms of success rates of worst-case MIAs~\cite{kulynych2024attack}. Consider the following MIA setting in which an adversary with access to a partial dataset $\bar D$, the target record $\xt$, and a model $\mt$, aims to tell whether $\mt$ comes from $\train(\bar D)$ or $\train(\bar D \cup \{ x \})$:
\begin{equation}\label{eq:hypothesis-test-def}
    H_0: \mt \sim \train(\bar D) \quad H_1: \mt \sim \train(\bar D \cup \{\xt\}).
\end{equation}
We omit the analogous case of $H_0$ corresponding to $\train(\bar D \cup \{x\})$ and $H_1$ to $\train(\bar D)$ for simplicity of exposition. 
Given a \emph{distinguisher} $\phi: \Theta \rightarrow [0, 1]$ which outputs $1$ to guess the membership of $\xt$ in the training dataset ($H_1$), and $0$ to guess its non-membership ($H_0$), we can characterize its success by its false positive rate (FPR) $\alpha_\phi$ and false negative rate (FNR) $\beta_\phi$:
\begin{equation}
    \label{eq:hypothesis-test}
    \begin{aligned}
    \alpha_\phi = \E_{\mt \sim \train(\bar D)}[\phi(\mt)], \quad
    \beta_\phi = 1 - \E_{\mt \sim \train(\bar D \cup \{x\})}[\phi(\mt)]
    \end{aligned}
\end{equation}
To analyze the privacy guarantees within this setting, we can consider the worst-case distinguisher $\phi^*_\alpha$ which achieves the lowest FNR at a given level of FPR $\alpha$:
\begin{equation}
    \phi^*_\alpha = \arg \inf_{\phi:~\Theta \rightarrow [0, 1]} \{\beta_\phi ~\mid~ \alpha_\phi \leq \alpha \}.
\end{equation}
Such an optimal attack always exists and can be constructed via Neyman-Pearson's lemma~\cite{dong2022gaussian}. 
% Denoting by the function $f(\alpha) = \beta_{\phi^*_\alpha}$ the FNR of the optimal attack that achieves FPR $\alpha$.
An algorithm $\train(\cdot)$ satisfies $(\varepsilon, \delta)$-DP if and only if the FNR of the optimal attack is lower bounded as follows:
\begin{equation}
    \beta_{\phi^*_\alpha} \geq \min \{0,\ 1 - e^\varepsilon\alpha - \delta,\ e^{-\varepsilon}(1 - \alpha - \delta)\},
\end{equation}
for any given level of FPR $\alpha \in [0, 1]$, any $\bar D \in 2^\dataspace$ and $\xt \in \dataspace$~\cite{dong2022gaussian}.

We refer to the trade-off curve, i.e., the set of all attainable $\alpha_\phi, \beta_\phi$, which is equivalent to the ROC curve of the worst-case MIA, as the \emph{differential privacy distinguisher (DPD) risk} of attack $\phi(\cdot)$, following the prior terminology~\cite{salem2023sok}. As DP is a standard notion of privacy leakage in statistical learning, we consider DPD risk an appropriate measure of privacy risk in our settings.

\section{Record-specific MIA evaluation}\label{sec:recordspecificeval}
In this section, we formalize the traditional and model-seeded privacy games and their estimations of attacker success.
\subsection{Traditional privacy game}

We refer to the privacy game commonly used for record-specific evaluation for adversaries under realistic assumptions as the \textit{traditional} game~\cite{ye2022enhanced, stadler2022synthetic, houssiau2022tapas, meeus2023achilles, guepin2023synthetic}. In this privacy game, an attacker's success in inferring a target record's membership is evaluated over multiple runs of the attack, each using a freshly sampled target dataset and the same target record $\xt$. 
We formalize this game and later analyze the resulting risk estimate, which we denote as $\ravg$.

\begin{definition}[Traditional record-specific privacy game]
    For target record $\xt$, dataset size $n$, training algorithm $\train(\cdot)$, and attack $\phi(\cdot)$:
    \begin{enumerate}
        \item The challenger samples dataset $\Bar{D} \sim \datagen^{n}$ from the distribution with a fresh random seed.
        \item The challenger draws a secret bit $b\in\{0,1\}$ uniformly at random with a fresh random seed.
        \item If $b=1$, the challenger adds target record $\xt$ to dataset $\Bar{D}$ to form the target dataset $\dt = \Bar{D}\cup\{\xt \}$. Otherwise, $\dt = \bar D$.
        \item The challenger trains the target model $\mt \leftarrow \train(\dt)$ on dataset $\dt$ with a fresh random seed.
        \item The adversary outputs a guess $\hat{b} = \phi(\mt)$.
    \end{enumerate}
\end{definition}

\subsection{Model-seeded privacy game}

We now formalize the \textit{model-seeded} privacy game. The attacker's success is averaged across multiple runs of the attack, each using the same target record and target dataset, and sampling a fresh seed for training the target model. We denote the resulting risk estimate as $\rsp$.

\begin{definition}[Model-seeded record-specific privacy game]\label{def:ms_privacygame}For target record $\xt$, partial dataset $\bar D$, training algorithm $\train(\cdot)$, attack $\phi(\cdot)$,  and number of runs $N$:
    \begin{enumerate}
        \item The challenger draws a secret bit $b\in\{0,1\}$ uniformly at random with a fresh random seed.
        \item If $b=1$, the challenger adds target record $x$ to $\Bar D$ to form the target dataset: $\dt = \bar D \cup \{ \xt\}$. Otherwise, $\dt = \bar D$.
        \item The challenger trains the target model $\mt \leftarrow \train(\dt)$ on dataset $\dt$ with a fresh random seed.
        \item The adversary outputs a guess $\hat{b} = \phi(\mt)$.
    \end{enumerate}
\end{definition}

In contrast with the traditional privacy game, this game results in an estimate of the target record's risk within a specific target dataset. By using only the model seed as a source of randomness and eliminating dataset sampling, the ability of the MIA to infer the presence of the target record $\xt$ in models trained on $\dt$ is evaluated. To the best of our knowledge, this privacy game has so far been used exclusively to evaluate the worst-case attack, where the adversary is assumed to have full knowledge of the target dataset apart from the membership of the target record~\cite{annamalai2024theory, ganev2025elusive}, and never used to evaluate adversaries under realistic assumptions with access only to auxiliary data.

\subsection{The relationship between the games and privacy risk}\label{subsec:theoretical_result}
In this section, we analyze the relationship between success rates of MIAs as evaluated in the games introduced previously, and the DPD risk defined in \cref{sec:dp-risk}.
Consider $N > 1$ runs of either the model-seeded or traditional game with different random seeds, resulting in a set of guesses $\{ \hat b_i \}_{i \in [N]}$ with corresponding secret bits (i.e. membership labels) $\{ b_i \}_{i \in [N]}$. Let us denote the empirical FPR and FNR obtained in an evaluation using a privacy game for a given attack $\phi: \modelspace \rightarrow [0, 1]$:
\begin{align}
    \hat \alpha_\phi &= \frac{ \sum_{i=0}^{N} \mathbbm{1} \{ \hat b_i = 1 \land b_i = 0 \} }{ \sum_{i=0}^{N} \mathbbm{1} \{ b_i = 0 \} }, \\
    \quad \hat \beta_\phi &= \frac{ \sum_{i=0}^{N} \mathbbm{1} \{ \hat b_i = 0 \land b_i = 1 \} }{ \sum_{i=0}^{N} \mathbbm{1} \{ b_i = 1 \} }
\end{align}
We use $\hat \alpha^\text{T}_\phi$ or $\hat \alpha^\text{MS}_\phi$ to denote the empirical error rates computed using the traditional (T) and  model-seeded (MS) game, respectively. We use $\hat \beta^\text{T}_\phi$ and $\hat \beta^\text{MS}_\phi$ analogously.
We show that, with a sufficiently large number of repetitions of the game with freshly drawn seeds, the empirical FPR and FNR obtained using the model-seeded game converge exponentially fast to the DPD risk as defined in \cref{sec:dp-risk} for any given attack $\attack$ and record $x$:
\begin{restatable}[Model-seeded game converges to DPD risk]{proposition}{modelseededconvergence}\label{proposition:ms}

For any fixed target record $x$, partial dataset $\bar D \in \dataspace^{n-1}$, training algorithm $T(\cdot)$, and attack $\phi(\cdot)$, we have w.p. $1 - \rho$ for $\rho \in (0, 1)$ over $N$ random coin flips, i.e., fresh seed draws, in the model-seeded game:
    \begin{equation}
        \begin{aligned}
        |\alphasp - \alpha_\phi| \leq \sqrt{\frac{\log(2 / \rho)}{2N}}, \\ |\betasp - \beta_\phi| \leq \sqrt{\frac{\log(2 / \rho)}{2N}} \\
        \end{aligned}
    \end{equation}
\end{restatable}
This is based on observing that the model-seed game effectively samples from the probability distributions corresponding to $H_0$ and $H_1$ in \cref{eq:hypothesis-test-def} and consequently applying applying standard Hoeffding's inequality. We provide a detailed proofs of this and the following formal statements in \cref{appendix:proofs}.

In contrast, empirical error rates obtained using the traditional privacy game do not converge to the respective success rates of the attacks, but rather to their \emph{average} over i.i.d. dataset resamples:
\begin{restatable}[Traditional game converges to average privacy risk]{proposition}{avgdatasetconvergence}\label{proposition:tr}%
For any fixed target record $x$, dataset size $n > 1$, training algorithm $T(\cdot)$, and attack $\phi(\cdot)$, we have w.p. $1 - \rho$ for $\rho \in (0, 1)$ over $N$ random coin flips, i.e., fresh seed draws, in the traditional game:
    \begin{equation}
        \begin{aligned}
        |\alphaavg - \E_{\bar D \sim \datagen^{n}} \alpha_{\phi, \bar D}| \leq \sqrt{\frac{\log(2 / \rho)}{2N}}, \\ |\betaavg - \E_{\bar D \sim \datagen^{n}} \beta_{\phi, \bar D}| \leq \sqrt{\frac{\log(2 / \rho)}{2N}}, \\
        \end{aligned}
    \end{equation}
    where we explicitly use $\alpha_{\phi, \bar D}$ and $\beta_{\phi, \bar D}$ to emphasize the dependence of $
\alpha_\phi$ and $\beta_\phi$ on $\bar D$ in the definition of the hypothesis test in \cref{eq:hypothesis-test}.
\end{restatable}
Thus, the model-seeded game serves as an estimator of the DPD risk of an attack, as opposed to the traditional game, which estimates an average risk over hypothetical dataset re-samples.

\section{Experimental setup}\label{sec:experimentalsetup}
\subsection{Data, target model and attack details}\label{subsec:expsetup_data}

\paragraph{Datasets.}
We use the Adult~\cite{misc_adult_2} and UK Census~\cite{census2011} datasets. Both are anonymized samples of census data containing categorical and continuous demographic features. We partition each dataset into $D_{\text{aux}}$, used for MIA development, and $D_{\text{eval}}$, used for evaluation. We consider $|\dt|=1000$, $\dt\subset D_{\text{eval}}$. The sizes of the auxiliary dataset and evaluation pool are reported in~\cref{appendix:exp_details}.

\paragraph{Target models.}

We use Synthpop~\cite{synthpop} and Baynet~\cite{baynet} in our main experiments, and PrivBayes~\cite{baynet}, the differentially private version of Baynet, for our experiments training the target model with DP guarantees (\cref{sec:dp}). For all models, we use the implementations available in the reprosyn~\cite{reprosyn2022} repository.

\paragraph{MIA methodology.}

We use the extended version of TAPAS, the state-of-the-art query-based attack for SDGs, as originally introduced by~\citet{houssiau2022tapas}, and extended by~\citet{meeus2023achilles}. We train the attack for each target record using auxiliary dataset $D_{\text{aux}}$ to sample shadow datasets. The number of shadow models $N_{\text{shadow}}$ and other development configurations we use are specified in \cref{appendix:exp_details}. TAPAS operates under black-box model access with auxiliary data, but no access to the training data of the target model. Throughout the evaluation, we use AUC ROC as a summary metric for privacy risk.

\subsection{Practical implementation of the privacy games}

We outline in~\cref{alg:spec_testing} our practical implementation of the traditional and model-seeded privacy games. In both setups, we construct $\frac{N_{\text{eval}}}{2}=500$ `in' and `out' datasets each. We validate our choice of $N_{\text{eval}}$ in~\cref{appendix:convergence}. In the traditional setup, we sample `out' datasets of size $|\dt|$ from the evaluation pool $D_{\text{eval}}$, and ensure $\xt$ is included in exactly half. In the model-seeded setup, the `in' datasets are equivalent to the full target dataset $\dt$. To maintain an equal dataset size across runs, we construct the `out' datasets by replacing $\xt$ with a randomly sampled reference record  \(x_r\sim \mathsf{Unif}[D_{\text{eval}} \setminus D] \).

\begin{algorithm*}
    \caption{Practical privacy game implementation}
    \label{alg:spec_testing}
    \textbf{Input:} Target record \(\xt\), target dataset \(\dt\), evaluation pool \(D_{\text{eval}}\), partial evaluation pool \(\Bar D_{\text{eval}}=D_{\text{eval}} \setminus \{\xt\}\), MIA $\phi_x$ with membership score function $\score_{\xt}$, attack thresholds $\gamma\in\{ \gamma_1, \ldots, \gamma_m \}$, number of runs \(N_{\text{eval}}\), and privacy game flag $\mathrm{PG}\in\{ T, MS \}$. 
    $\mathrm{PG}=T$ indicates the traditional privacy game and $\mathrm{PG}=MS$ the model-seeded game. \\
    \textbf{Output:} Empirical error rates $[\hat \alpha^{PG}_{\attack, 1} , \ldots, \hat \alpha^{PG}_{\attack, m}] $ and $[\hat \beta^{PG}_{\attack, 1} , \ldots, \hat \beta^{PG}_{\attack, m}] $ for each attack threshold $\gamma$, and summary risk metric $R^{\mathrm{PG}}$ computed as the ROC AUC of attack $\phi_x$.
    \begin{algorithmic}[1] 
        \State{\comm{IN models}} 
        \For{$i = 0, 1, \ldots, \frac{N_{\text{eval}}}{2}$}
            \If{$\mathrm{PG}=MS$}
                \State \( D_{\text{in}}\leftarrow D \)
            \Else
                \State Sample $\bar D_{\text{in}}\sim \mathsf{Unif}[\Bar D_{\text{eval}}]^{|\dt| - 1} $
                \State \(D_{\text{in}} \leftarrow \bar D_{\text{in}} \cup \{\xt \} \)
            \EndIf
            \State Train evaluation model \(\theta_{\text{in}} \leftarrow T(\dt)\) with fresh random seed.
            \For{$\gamma_j \in \{ \gamma_1, \ldots, \gamma_m\}$}
                \State Compute attack prediction $\phi_x(\theta_{\text{in}}) = \mathbbm{1}[\score_{\xt}(\mt_{\text{in}}) \geq \gamma_j]$
                \State $ \hat b_{i,j} \leftarrow  \phi_x(\theta_{\text{in}})$
                % \State $s_i \leftarrow  \score_{\xt}(\mt_{\text{in}})$
                % \State Make prediction \(\hat{b}_i = \phi_x(\theta_{\text{in}})\)
            \EndFor
            \State $b_i \leftarrow 1$
        \EndFor
        \State{\comm{OUT models}} 
        \For{$i = \frac{N_{\text{eval}}}{2}, \ldots, N_{\text{eval}}$}
            \If{$\mathrm{PG}=MS$}
                \State Sample reference record \(x_r\sim \mathsf{Unif}[D_{\text{eval}} \setminus D] \).\label{line:ref1}
                \State Remove $\xt$ from $\dt$ and replace with $x_r$ to construct dataset $D_{\text{out}}$.\label{line:ref2}
            \Else
                \State Sample $\bar D_{\text{out}} \sim \mathsf{Unif}[\Bar D_{\text{eval}}]^{|\dt|} $
            \EndIf
            \State Train evaluation model $\theta_{\text{out}} \leftarrow T(D_{\text{out}})$ with fresh random seed.
            \For{$\gamma_j \in \{ \gamma_1, \ldots, \gamma_m\}$}
                \State Compute attack prediction $\phi_x(\theta_{\text{out}}) = \mathbbm{1}[\score_{\xt}(\mt_{\text{out}}) \geq \gamma_j]$
                \State $ \hat b_{i,j} \leftarrow  \phi_x(\theta_{\text{out}})$
                % \State $s_i \leftarrow  \score_{\xt}(\mt_{\text{out}})$
                % \State Make prediction \(\hat{b}_i = \phi_x(\theta_{\text{out}})\)
            \EndFor
            \State $b_i \leftarrow 0$
            \EndFor
        \State Compute empirical FPR $\hat \alpha_{\attack, j}^{PG}( \{ \hat b_{i,j} \}_{i = 1}^{N_{\text{eval}}}, \{ b_i \}_{i=1}^{N_{\text{eval}}} )$ for each $\gamma_j\in [\gamma_1,\ldots,\gamma_m]$
        \State Compute empirical FNR $\hat \beta_{\attack, j}^{PG}( \{ \hat b_{i,j} \}_{i = 1}^{N_{\text{eval}}}, \{ b_i \}_{i=1}^{N_{\text{eval}}} )$ for each $\gamma_j\in [\gamma_1,\ldots,\gamma_m]$
        \State Compute summary privacy risk  \( R^{\mathrm{PG}} = \mathrm{AUC}( \{ \hat \alpha_{\attack, j}^{PG}\}_{j=1}^m, \{ \hat \beta_{\attack, j}^{PG}\}_{j=1}^m ) \)
    \end{algorithmic}
\end{algorithm*}

\subsection{Comparison metrics}\label{subsec:comparison_metrics}

We use the following metrics to compare the traditional and model-seeded  risk estimates. 

\paragraph{Miss rate.} We define miss rate as the fraction of records classified as high-risk in the model-seeded setup, which are classified as low-risk in the traditional setup. For a subset of records in the target dataset $S\subseteq\dt$ for which we compute the risk estimates and high-risk threshold $t$, we compute the miss rate as: 
\begin{align}
    \mathrm{MR}(S) =  
    \frac{
    |\{  x\in S ~\mid~ \ravg(x) \leq t ~\wedge~ \rsp(x) >t \}|
    }{
    |\{  x\in S ~\mid~ \rsp(x) >t \}|
    }.
\end{align}

\paragraph{Root Mean Squared Deviation (RMSD).} For $S\subseteq\dt$, we compute the RMSD between the two risk estimates as
\begin{align}
    \mathrm{RMSD}(S) = \sqrt{\frac{1}{|S|} \sum_{x\in S} \left( \ravg(x) - \rsp(x)  \right)^2}.
\end{align}
RMSD measures the difference between traditional and model-seeded risks, capturing error magnitude regardless of whether traditional risk over- or underestimates model-seeded risk.

\section{Difference between \(\rsp\) and \(\ravg\)}\label{sec:mainresult}
\begin{figure*}[t]
    \centering
    \resizebox{\figsize}{!}{
        \includegraphics[width=0.4\linewidth, page=1]{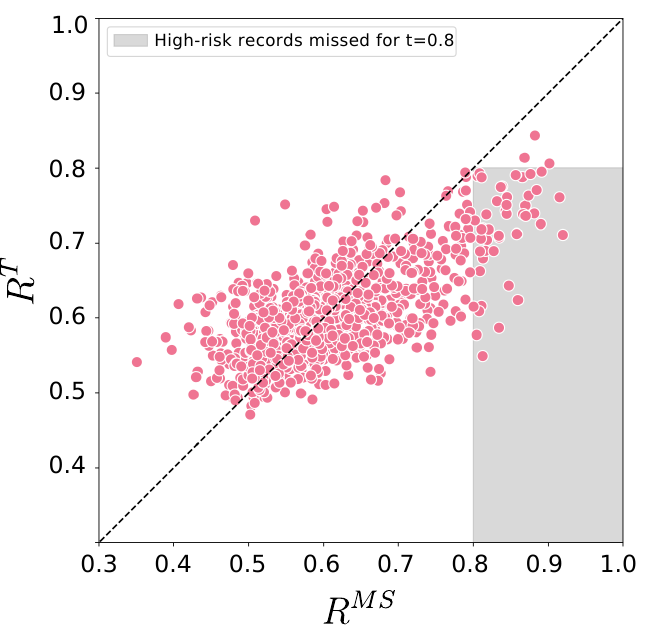}
        \includegraphics[width=0.375\linewidth]{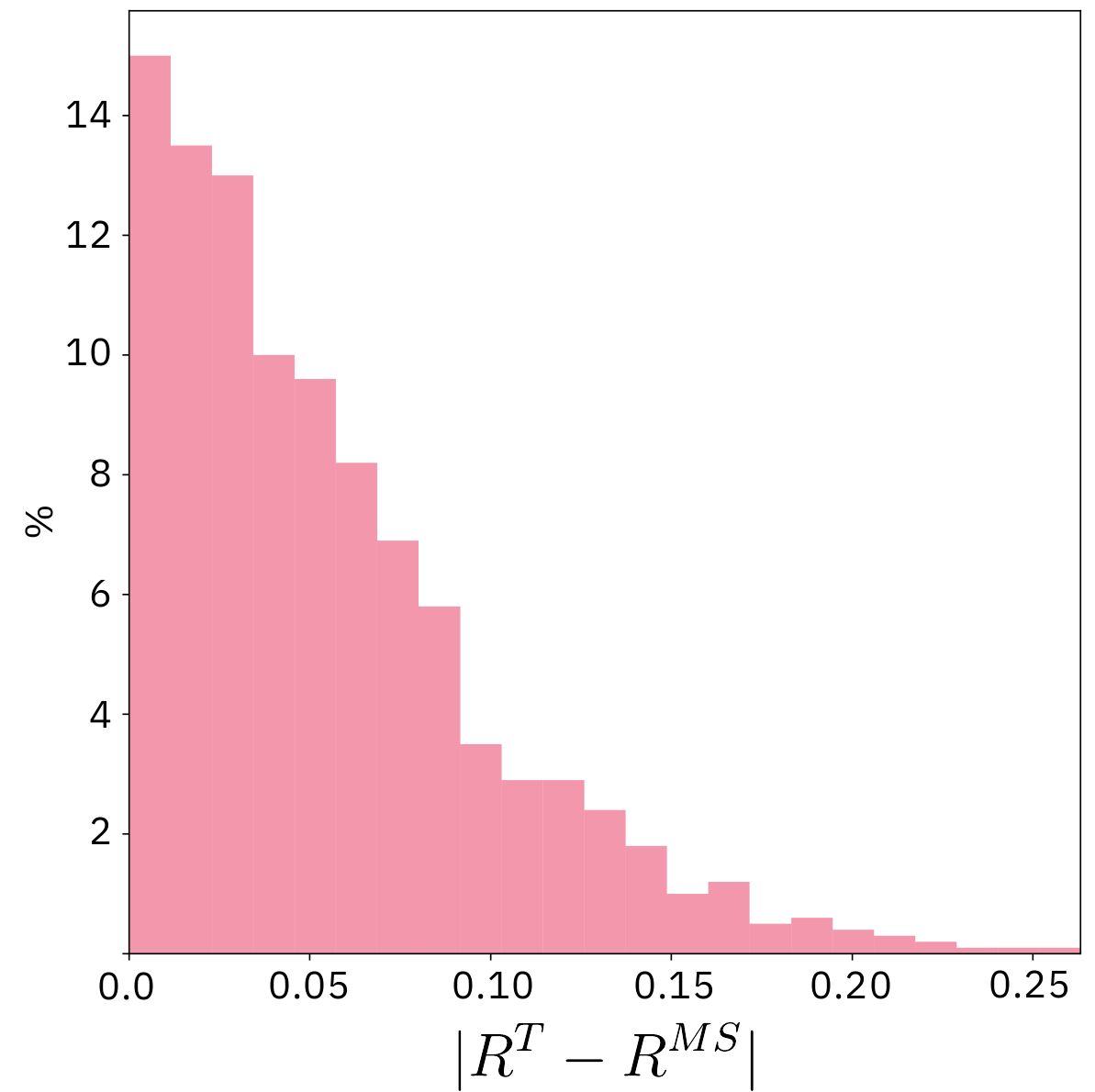}
    }
    \caption{Risk for all 1000 records in $\dt$ sampled from the Adult dataset (Synthpop). (a) per-record model-seeded and traditional risks. The shaded area marks all the high-risk records missed in the traditional setup for high-risk threshold $t=0.8$. (b) histogram of per-record absolute differences between the model-seeded and traditional risks. }
    \label{fig:adult}
\end{figure*}

We instantiate the TAPAS MIA across the datasets and SDGs described in \cref{sec:experimentalsetup} and evaluate them in the traditional and model-seeded setups. We then compare the resulting risk estimates using the metrics described in Section~\ref{subsec:comparison_metrics}.

Figure~\ref{fig:adult} shows the traditional and model-seeded risks to indeed differ substantially. For $\dt$ sampled from the Adult dataset and $\mt$ Synthpop, we compute the traditional and model-seeded risk estimates for all 1000 records in $\dt$ and present them in Figure~\ref{fig:adult}a. This shows that 94\% of high-risk records for high-risk threshold $t=0.8$ would indeed be incorrectly classified as low-risk when estimating risk in the traditional setup. Across all records, using the traditional setup leads to an RMSD of $0.07$, for a value that empirically ranges roughly from 0.5 to 1. Figure~\ref{fig:adult}b shows a histogram of absolute differences between the two risk estimates across records, showing that the estimate would be off by more than $0.1$ for $15\%$ of records when using the traditional setup, and could go up to $0.26$.
Table~\ref{tab:miss_rates_synthetic} shows that these results are consistent across SDG setups. Across both datasets and generators, the miss rates show that high-risk records are consistently being incorrectly identified. If the traditional setup were a good approximation of the risk we should observe miss rates close to 0. Instead, we observe values ranging from $0.73$ to $0.94$. The majority of the records that are highly vulnerable will thus be incorrectly considered low-risk if MIAs are evaluated using the traditional setup. We also obtain high RMSD across setups, ranging from $0.04$ to $0.11$, which represents a significant error for risk estimated using AUC.

\begin{table}[t]
    \caption{Miss rate across different datasets and target synthetic data generators. We use a high-risk threshold of $t=0.8$.}
    \centering
    \vspace{-1em}
    \begin{tabular}{cccc} \\ \toprule
         Dataset & Model & $\mathrm{RMSD}$ & $\mathrm{MR}$\\ 
          \midrule
         \multirow{2}{*}{Adult} & Synthpop & 0.07 &  0.94\\
          & Baynet & 0.05& 0.73 \\ \midrule
         \multirow{2}{*}{Census} & Synthpop  & 0.11 &  0.94\\
          & Baynet  &  0.04 & 0.75 \\  \bottomrule
    \end{tabular}
    \label{tab:miss_rates_synthetic}
\end{table}

\paragraph{Different high-risk threshold $t$ values.} We further examine how the choice of high-risk threshold $t$ influences the miss rate. \cref{fig:missrates_cdfs}a shows that, for all high-risk thresholds, the miss rates are substantial, reaching values above 20\% for all setups for $t=0.6$ and up to 80\% for $t=0.9$. Using the traditional setup instead of the model-seeded setup for MIA evaluation thus leads to high-risk records being incorrectly classified as low-risk, regardless of the threshold choice.

Notably, we further find that the miss rate increases with larger threshold values $t$. As the criterion for classifying records as high-risk becomes more strict, identifying such records becomes more difficult, and the traditional setup fails to detect an increasing fraction of them. 

\begin{figure*}[t]
  \centering
  \resizebox{\figsize}{!}{
      \includegraphics[width=0.45\linewidth]{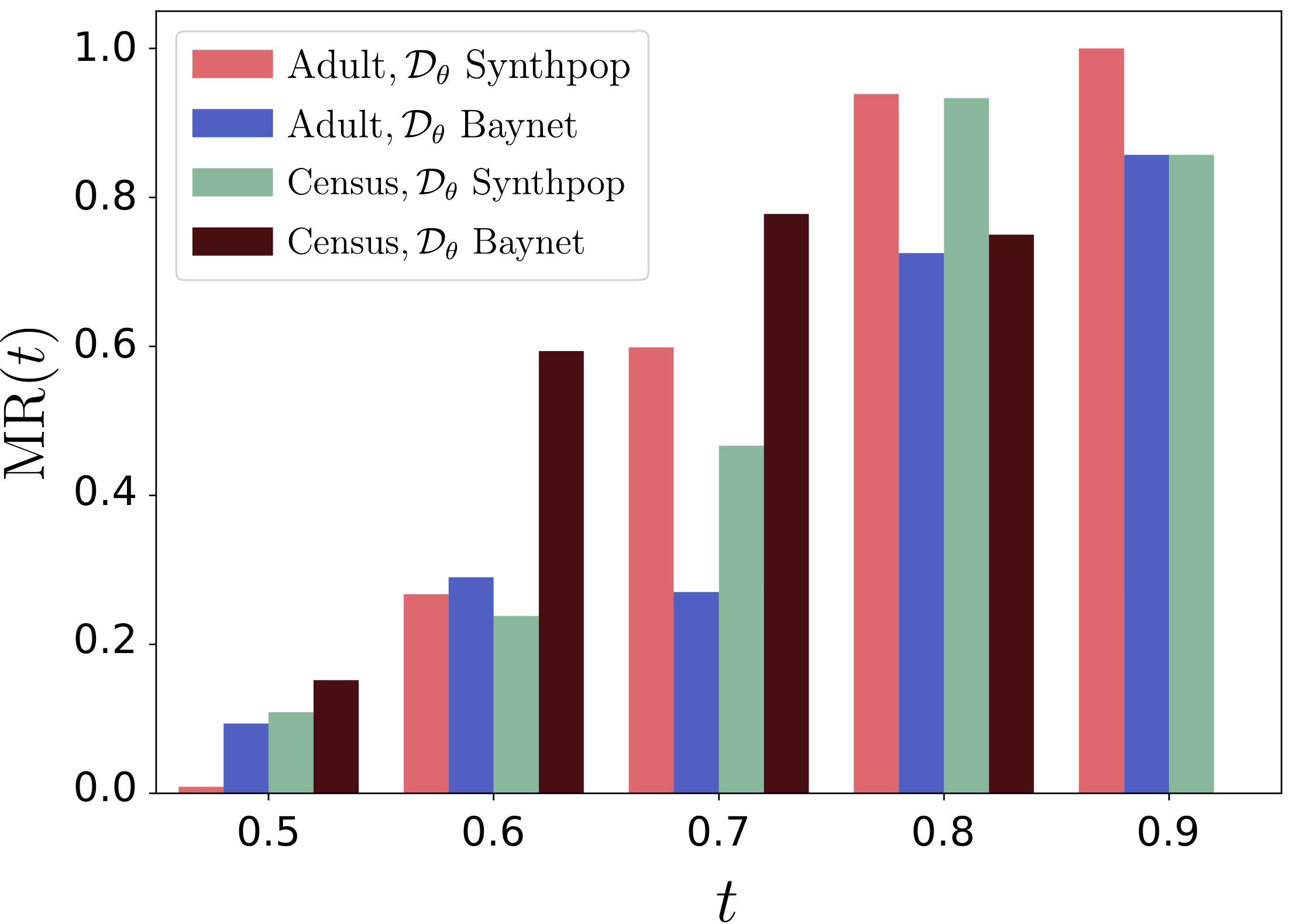}
      \includegraphics[width=0.42\linewidth]{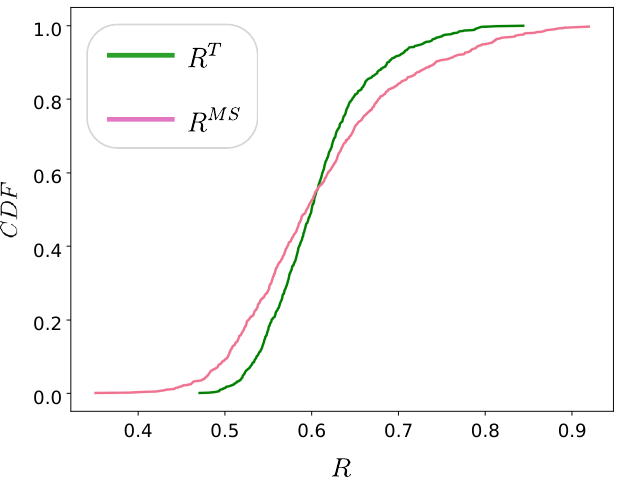}
  }
  \caption{(a) Miss rate for different high-risk thresholds $t$ for SDG setups. Note that for Census and Baynet, there are no records with $\rsp>0.9$, therefore the miss rate is not defined. (b) CDFs of $\rsp$ and $\ravg$ across records in target dataset $\dt$ drawn from Adult dataset with $\mt$ Synthpop.}
  \label{fig:missrates_cdfs} 
\end{figure*}

\paragraph{Risk distribution across records.} 
Figure~\ref{fig:missrates_cdfs}b further shows the cumulative distribution of record-level risk estimates under both setups. Risks estimated using the traditional setup are more narrowly distributed, with fewer records assigned very low or very high risk estimates compared to the model-seeded setup. This is particularly meaningful in the high-risk tail, as the traditional setup underestimates worst-case risk for all records in the dataset. For instance, in the Adult dataset, the 90th percentile of traditional risk is 0.68, compared to 0.74 under the model-seeded setup, indicating a substantially higher worst-case risk that the traditional setup fails to capture.

\paragraph{Dataset size.}\label{sec:ablation}

In~\cref{sec:mainresult}, we consider target datasets $\dt$ of size $1000$, a commonly used setup for MIAs against tabular SDGs~\cite{stadler2022synthetic, meeus2023achilles, guepin2023synthetic}. We now study how the size of the target dataset influences the gap between traditional and model-seeded risks. For 20 records from the Adult dataset and with target model Synthpop, we compute the model-seeded and traditional risks for target datasets $\dt$ ranging in size $|\dt|=n$ for $n \in \{ 200, \ldots, 10000 \}$.

\begin{figure*}
  \centering
  \resizebox{\figsize}{!}{
  \includegraphics[width=\linewidth]{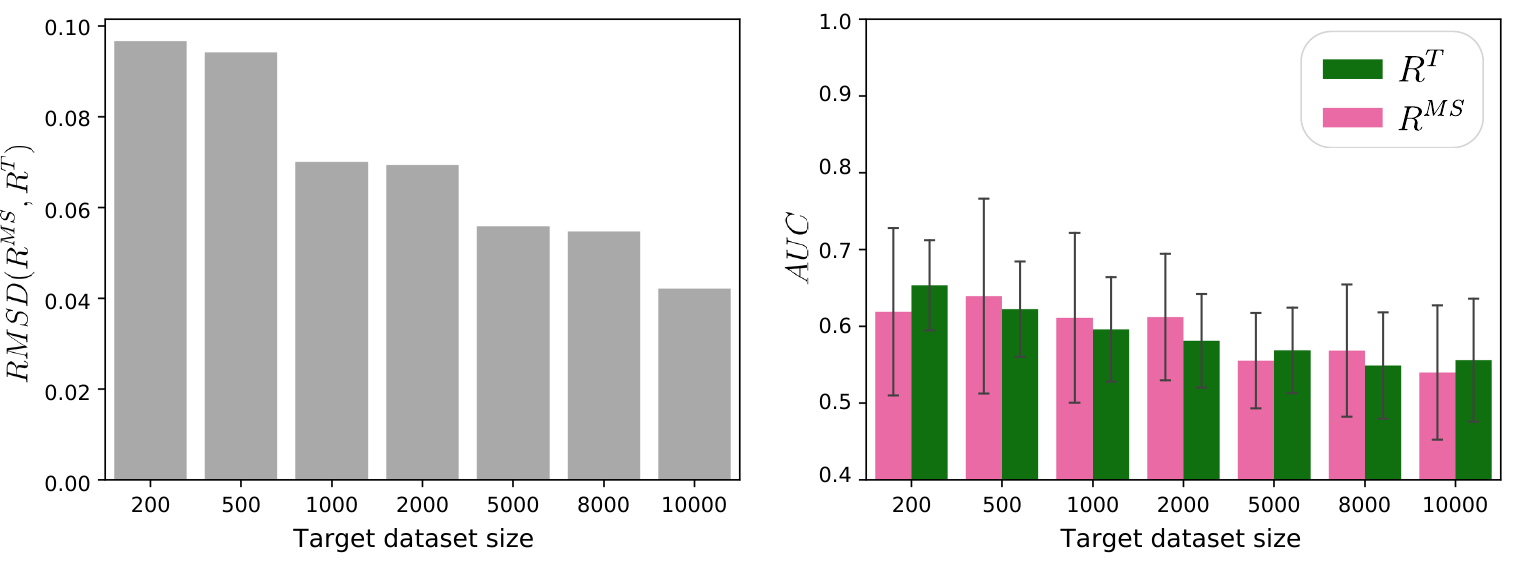}
  }
  \caption{(a) RMSD between model-seeded and traditional risk per target dataset size. (b) Model-seeded and traditional risk values per target dataset size. For both figures, values are computed across $20$ target records.}
  \label{fig:adult_ablation} 
\end{figure*}

\cref{fig:adult_ablation}a shows the RMSD between the two risk estimates across the $20$ target records for different target dataset sizes $|\dt|$. The error decreases for larger $\dt$, but is significant even for  $|\dt|=10,000$.~\cref{fig:adult_ablation}b shows the MIA AUC computed in both setups, averaged across the target records. MIA performance decreases---though it remains better than random---for larger datasets, naturally decreasing the gap between the risk estimates. Yet, highly vulnerable records are present even in large datasets, and the traditional and model-seeded risks do not converge to the same values, showing the importance of using the model-seeded game regardless of dataset size.

\paragraph{Training with formal privacy guarantees.}\label{sec:dp}

The standard way to protect against MIAs is by ensuring formal privacy guarantees, in particular by training SDGs with differential privacy (DP). Differential privacy provides a theoretical upper bound for attacker success, and MIAs can then be used to compute empirical lower bounds and validate DP implementations~\cite{annamalai2024theory, jagielski2020auditingdifferentiallyprivatemachine}. We now analyze the effect of training with DP with varying $\varepsilon$ on the model-seeded and traditional risks. For $20$ target records sampled from the Adult dataset, we compute the traditional and model-seeded risks for $\mt$ Privbayes with $\varepsilon\in\{ \infty, 100, 10, 1 \}$, and study their values and the resulting difference between then.

\cref{fig:dp_scatterplot} shows that MIA performance decreases for stricter values of $\varepsilon$, converging to the random guess baseline (AUC of $0.5$). As the estimates become concentrated around $0.5$, the difference between them also decreases, indicating that training with stricter privacy guarantees may lead to the specific target dataset having less impact on the individual risk of each record. 

We argue that in the case of training with DP guarantees, it is also important to use the model-seeded setup to evaluate MIAs. Indeed, when validating DP guarantees, estimating the risk on average rather than for the actual training data may mask a violation, potentially leading to undetected privacy leakage.

\begin{figure*}[ht]
  \centering
  \includegraphics[width=\linewidth]{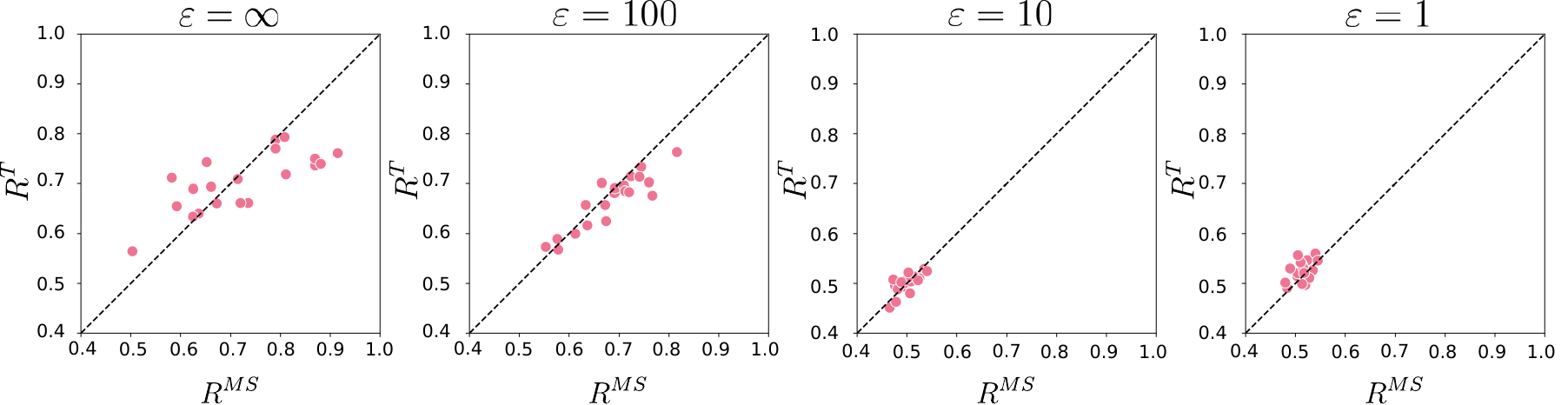}
  \vspace{-2em}
  \caption{Model-seeded and traditional risk for $20$ records for Adult dataset and $\mt$ Privbayes with varying $\varepsilon$.}
  \label{fig:dp_scatterplot}
\end{figure*}

\paragraph{Evaluating one record's risk within different datasets.}
We use the Adult dataset and the Synthpop model to illustrate an example of the potential negative impact of using the traditional instead of the model-seeded setup. We compute the risk of a single target record in the traditional setup. Then, we compute its model-seeded risk in 15 randomly selected datasets sampled in the traditional setup. As shown in~\cref{fig:rugplot}, the model-seeded risk $\rsp$ varies from approximately 0.5 (random guess) to 0.8 (high risk), depending on the dataset. The traditional risk is $\ravg = 0.62$, underestimating the DPD risk by up to 0.2 in the worst case.

\begin{figure*}[ht]
    \centering
    \resizebox{.6\linewidth}{!}{
    \includegraphics[width=\linewidth]{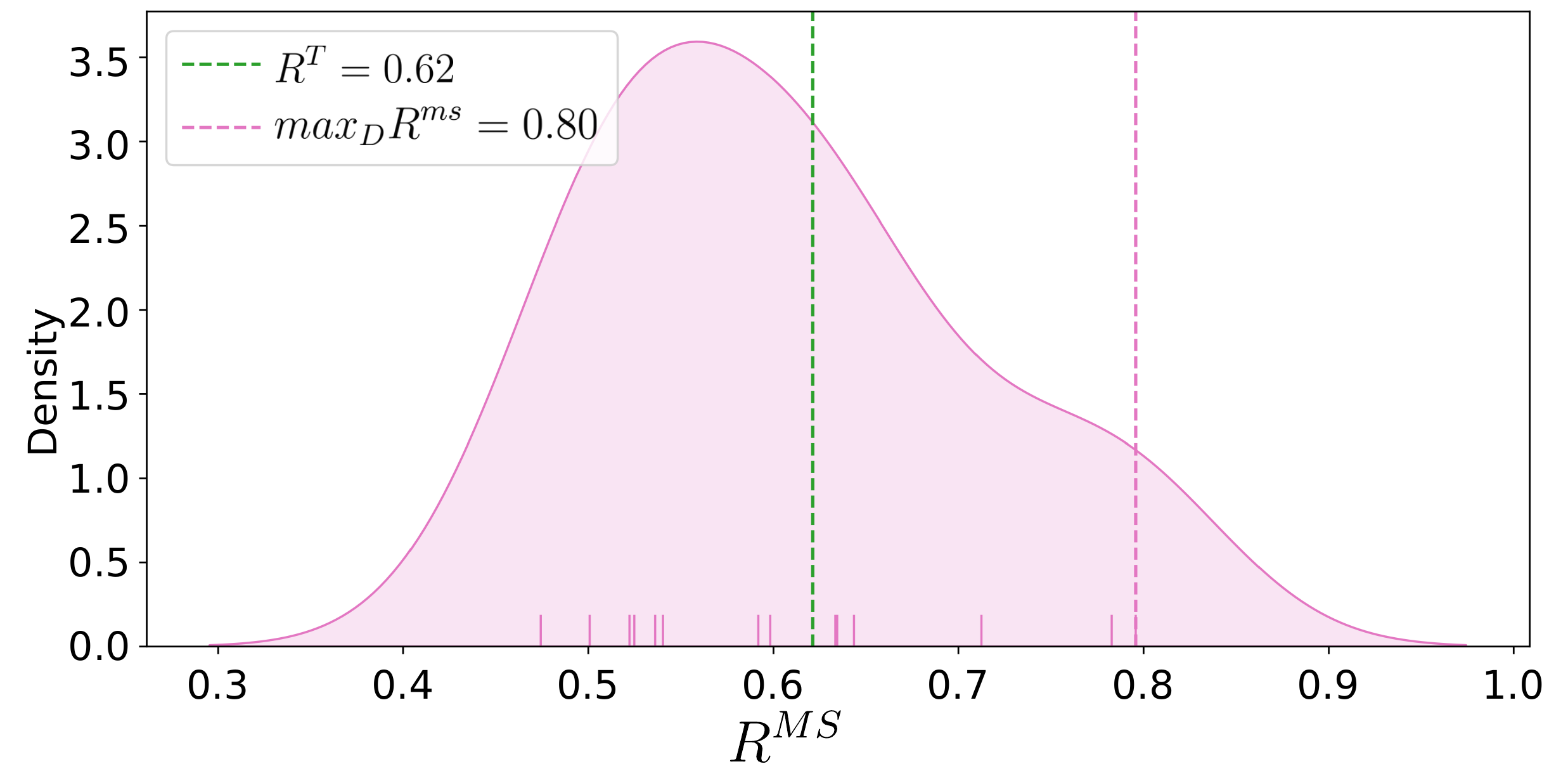}
    }
    \caption{Model-seeded risks of one target record within 15 different datasets and its traditional risk.}
    \label{fig:rugplot}
\end{figure*}

\section{Related work}\label{sec:relatedwork}
\paragraph{Membership inference attacks (MIAs).}

\citet{shokri2017membership} introduced the first MIA against ML models, using the target model's predictions to infer membership of a given record. Their approach relies on the \textit{shadow modeling} technique where multiple models including and excluding the target record are trained to approximate its impact on the final model. Various attacks based on shadow modeling have since been proposed~\cite{sablayrolles2019white, ye2022enhanced, song2020systematicevaluationprivacyrisks, zarifzadeh2024lowcosthighpowermembershipinference, watson2021importance}, typically leveraging the model loss computed for the target record as the primary signal to infer membership. The current state-of-the-art attack for ML models, introduced by \citet{carlini2022membership}, performs a likelihood ratio test between the loss distribution of models trained on the target record and those not.

Tabular synthetic data generators generally model a dataset as a whole, learning feature distributions and sampling synthetic records from them~\cite{synthpop,baynet, 10.1145/3085504.3091117}. As a result, they do not have a notion of per-record loss, and standard MIAs designed for ML models cannot be directly applied to them. Instead, specialized attacks that rely on generated data and are developed for individual records have recently been proposed, using shadow modeling to capture the impact of a specific target record on the synthetic data~\cite{stadler2022synthetic, meeus2023achilles, houssiau2022tapas}. \citet{stadler2022synthetic} compute statistical features on synthetic data generated by shadow models and use them to train a meta-classifier to predict membership of the target record. \citet{houssiau2022tapas} propose TAPAS, extending the attack of Stadler et al. to include $k$-way queries which count the number of synthetic records with identical values to the target record for random subsets of $k$ features. \citet{meeus2023achilles} extend the $k$-way queries of TAPAS to include less-than-or-equal-to queries for continuous columns.

\paragraph{Threat models.}
A threat model defines the setup of an attack by specifying the attacker's assumed access to the target model and data. Model access can be black-box or white-box. For attacks against tabular synthetic data, black-box adversaries are most common, where the attacker knows the model architecture and can either query the model an arbitrary number of times~\cite{houssiau2022tapas}, or has access to a fixed set of generated data~\cite{meeus2023achilles, stadler2022synthetic}. Black-box attacks are widely studied for predictive models, where the attacker has access to predicted probabilities~\cite{carlini2022membership, zarifzadeh2024lowcosthighpowermembershipinference, song2020systematicevaluationprivacyrisks} or labels~\cite{choquette2021label}. White-box adversaries also know the internal parameters of the model, and are more common for image generators and some predictive ML models~\cite{hayes2019logan, pang2024whiteboxmembershipinferenceattacks, azadmanesh2021whitebox, matsumoto2023membership, cretu2023investigating}. 

Data access defines the data available to the attacker and its relationship to the target data. Adversaries for both tabular synthetic data and predictive models are usually assumed to have access to an auxiliary dataset drawn from the same distribution as the target data~\cite{stadler2022synthetic, meeus2023achilles, houssiau2022tapas, carlini2022membership}. \citet{guepin2023synthetic} show that this assumption can be relaxed by replacing the auxiliary dataset with synthetic data, albeit with reduced attack performance. The literature on privacy auditing considers a very strong \textit{leave-one-out} attacker, who has knowledge of the full target dataset apart from the membership of the target record~\cite{jagielski2020auditingdifferentiallyprivatemachine, nasr2021adversaryinstantiationlowerbounds, tramer2022debuggingdifferentialprivacycase}.

\paragraph{MIA evaluation.}
MIAs are commonly evaluated using a \emph{privacy game} between an attacker and a challenger~\cite{carlini2022membership, stadler2022synthetic, yeom2018privacy, pyrgelis2018knock,jayaraman2021revisiting}. 
%A performance score is computed using metrics such as the area under the receiver operating characteristic curve (ROC AUC) \cite{shokri2017membership}, attack accuracy~\cite{stadler2022synthetic}, or the true positive rate (TPR) at low false positive rates (FPR) \cite{carlini2022membership}.
\citet{ye2022enhanced} provide a comprehensive overview of privacy games used to evaluate MIAs, distinguishing between \emph{model-specific} and \emph{record-specific} privacy games. In model-specific games, the attacker attempts to distinguish between records included or excluded from the training data of one target model. They are commonly used to evaluate attacks based on the model's prediction confidences~\cite{carlini2022membership, zarifzadeh2024lowcosthighpowermembershipinference, meeus2025canarysechoauditingprivacy, song2020systematicevaluationprivacyrisks, matsumoto2023membership}. Record-specific privacy games evaluate the ability of the attacker to distinguish between models trained on the target record and those not. The record-specific privacy game for an average dataset (our \textit{traditional} privacy game) is typically used to evaluate attacks for synthetic tabular data~\cite{houssiau2022tapas, meeus2023achilles, stadler2022synthetic}. 

\citet{ye2022enhanced} define a privacy game for a \textit{fixed worst-case record and dataset} which is used in existing literature to test differential privacy guarantees with a very strong \textit{leave-one-out} attacker~\cite{annamalai2024theory, jagielski2020auditingdifferentiallyprivatemachine, nasr2021adversaryinstantiationlowerbounds, tramer2022debuggingdifferentialprivacycase}. To the best of our knowledge, this privacy game has never been used for weaker attackers or attacks against synthetic data. The model-seeded game can be used to evaluate the worst-case attacker, but is applicable to any attack, regardless of assumptions. Specifically, the goal of the model-seeded game is to measure the DPD risk for \emph{any} attacker, rather than only the worst-case attacker.

\section{Discussion}\label{sec:discussion}
We show that the model-seeded privacy game provides an unbiased estimate of a record's DPD risk, and that the traditional game averages the risk across datasets. Our empirical results show that the difference between the two games is meaningful in practice. Using the traditional instead of the model-seeded setup leads to an error of $0.11$ and $85\%$ of high-risk records being mistaken for low-risk on average across setups. This confirms that the assumption that the risk of a record is unaffected by the dataset it belongs to is optimistic, and that using the traditional game could lead to vulnerable records being overlooked. Though training with large target datasets or formal privacy guarantees decreases the difference between the two risk estimates, the model-seeded game provides a more appropriate estimate of the privacy risk, and should be used regardless. 

The exact nature of the impact of the dataset on a record's vulnerability is an open question. Prior work has suggested outliers--records in under-represented groups and records with rare feature values to be highly vulnerable~\cite{kulynych2022disparate, stadler2022synthetic, meeus2023achilles, carlini2022privacy, feldman2020does}. These factors are dataset-specific, and they may not be preserved across sampled datasets. For example, a record that is an outlier in one dataset may not be one in another, in particular in small or high-dimensional datasets. Very large datasets may better estimate the overall distribution of the data and preserve the outlier status of a record across datasets. The model-seeded and traditional risk may then be closer in value, as the datasets are overall more similar.

\section{Conclusion}\label{sec:conclusion}

In this work, we show that the privacy game typically used to evaluate record-specific attackers under realistic assumptions against SDGs disregards the specific dataset's impact on a record's risk, leading to potentially misleading risk estimates. We propose the model-seeded privacy game for evaluating any attacker, regardless of assumption, and show that it converges to the DPD risk of a record. We show that the difference between the two risk estimates is significant. We hope this work helps organizations handling sensitive data, such as in healthcare~\cite{lotan2020medical} and finance~\cite{fca_synth}, better assess data leakage risks and maintain high privacy standards when releasing synthetic data.

% \balance
% \bibliographystyle{plainnat}

% \begingroup
% \onecolumn
\renewcommand{\bibsection}{\section*{References}} % optional
\bibliographystyle{plainnat}
\bibliography{mybibliography}
% \endgroup

\onecolumn
\appendix
\section{Omitted formal results}\label{appendix:proofs}
% \modelseededunbiased*
% \begin{proof}[Proof of Proposition~\ref{proposition:ms}]
% \end{proof}
\modelseededconvergence*
\begin{proof}[Proof]
Consider the set of \emph{in} models $\{\theta_\text{in}^{(i)}\}_{i=1}^N$ and the set of \emph{out} models $\{\theta_\text{out}^{(i)}\}_{t=1}^N$ obtained in the model-seeded game. Let us define $X_i$ for $i \in [N]$ as follows:
\begin{align}
X_i &= \mathbbm{1}[\phi(\theta_\text{out}^{(i)}) = 1]
\end{align}
The set $\{X_i\}_{i=1}^N$ is a set of independent Bernoulli random variables. Let $\bar{X} = \frac{1}{N}\sum_{i=1}^{N}X_i$. Then $\hat \alpha = \bar{X}$. Moreover, we have that $\alpha = \E[\bar X]$, where the expectations are over sampling $\{\theta_\text{out}^{(i)}\}_{i=1}^N$.
By Hoeffding's inequality, for any $\gamma > 0$:
\begin{align}
\Pr[|\bar{X} - \mathbb{E}[\bar{X}]| \geq \gamma] &\leq 2e^{-2N\gamma^2}
\end{align}
% Thus, with probability at least $1 - 2e^{-2N\gamma^2}$:
% \begin{align}
% |\bar{X} - \mathbb{E}[\bar{X}]| < \gamma.
% \end{align}
Setting $2e^{-2N \gamma^2} = \rho$, we get:
\begin{align}
\gamma = \sqrt{\frac{\log(2/\rho)}{2N}},
\end{align}
which yields the sought statement. We get analogous results for $\beta$.

\end{proof}

The proof of~\cref{proposition:tr} is analogous to the proof of~\cref{proposition:ms}.

% \begin{proof}[Proof of~\cref{proposition:tr}]
% Analogous to \cref{proposition:ms}.
% \end{proof}

\section{Experimental details}\label{appendix:exp_details}

We report in~\cref{tab:sdg_params} the exact experimental details used in our main experiments.

\begin{table}[H]
    \centering
    \caption{Experimental details for synthetic data generator setups.}
    \label{tab:sdg_params}
    \begin{tabular}{cccccc}\toprule
         Dataset & $N_{\text{shadow}}$ & $N_{\text{eval}}$ & $|D_{\text{aux}}|$ & $|D_{\text{eval}}|$ & $|\dt|$  \\ \midrule
         Adult & 1000 & 1000 & 30000 & 15222 & 1000 \\
        Census & 1000 & 1000 & 52390 & 27193 & 1000 \\ 
         \bottomrule
    \end{tabular}
\end{table}

\section{Risk score convergence}\label{appendix:convergence}

To ensure that differences between $\rsp$ and $\ravg$ are not due to instability, we study the estimates as the number of evaluation models $N_{\text{eval}}$ increases. For 10 target records and $N_{\text{eval}}\in \{100, 200, 300, \cdots ,2000 \}$, we run the full evaluation pipeline for both setups, record the AUC, and repeat this 10 times, sampling new seeds each time. \cref{fig:stability}a shows the scores to stabilize around $N_{\text{eval}}=1000$. Figure~\ref{fig:stability}b shows the theoretical upper bound on the absolute error defined in Proposition~\ref{proposition:ms} to converge at a similar rate as the empirical scores.

\begin{figure}[H]
  \centering
  \includegraphics[width=\linewidth]{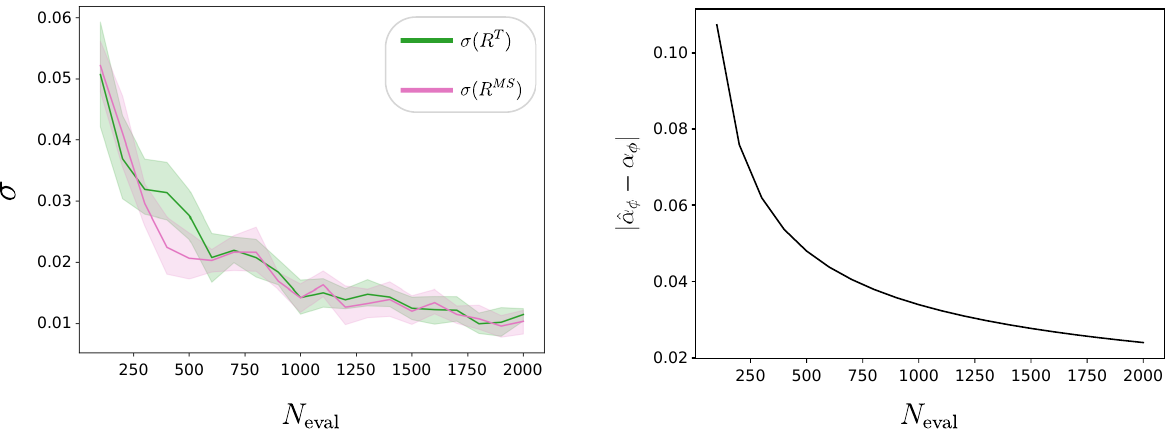}
  \caption{(a) Standard deviation of $R^{MS}$ and $R^{T}$ for different values of $N_{\text{eval}}$. (b) Theoretical upper bound on absolute error of risk estimate with $\rho=0.2$}
  \label{fig:stability} 
\end{figure}

\end{document}